%% file: main.tex
\documentclass{article} 
\usepackage{jmlr2e}

\usepackage[colorlinks,urlcolor=cyan,citecolor=blue,linkcolor=magenta]{hyperref} 
\usepackage{hyperref}
\usepackage{url}
\usepackage{physics}

\usepackage{colortbl}
\usepackage{booktabs}
\usepackage{multicol}
\usepackage[export]{adjustbox}
\usepackage{multirow}
\usepackage{tabularx}
\usepackage{makecell}
\usepackage[shortlabels,inline]{enumitem}

\usepackage{appendix}

\newcolumntype{Y}{>{\centering\arraybackslash}X}

\input{math_commands.tex}

\usepackage{amsmath,amssymb,amsthm,extarrows,mathrsfs}
\usepackage{url,here}
\usepackage{todonotes}
\usepackage{marginnote}
\usepackage{xpatch}
\makeatletter
\patchcmd{\@todonotes@drawMarginNoteWithLine}{\marginpar}{\marginnote}{}{}
\makeatother
\usepackage{ifpdf}
\usepackage{color}
\usepackage{bm}
\usepackage[all,cmtip]{xy}
\usepackage[normalem]{ulem}

\numberwithin{equation}{section}

\DeclareMathOperator{\supp}{supp}

\DeclareMathOperator{\LT}{LT}

\DeclareMathOperator{\spec}{Spec}

\DeclareMathOperator{\pr}{pr}

\def\Syz{\mathrm{Syz}}

\def\supp{\mathrm{supp}}

\newcommand{\gr}{Gr\"obner}

\newtheorem{heuristic}[theorem]{Heuristic}

\title{Geometric Generality of Transformer-Based Gröbner Basis Computation}

\author{
\name \hspace{-3mm} Yuta Kambe \email \texttt{\href{mailto:kambe.yuta@bx.mitsubishielectric.co.jp}{kambe.yuta@bx.mitsubishielectric.co.jp}} \\
       \addr 
       Mitsubishi Electric \\
       Kanagawa, Japan 
       \AND 
            \name Yota Maeda \email \texttt{\href{mailto:y.maeda.math@gmail.com}{y.maeda.math@gmail.com}} \\
       \addr 
       Technische Universität Darmstadt / Tohoku University\\
       Darmstadt, Germany / Sendai, Japan 
       \AND
       \name Tristan Vaccon \email \texttt{\href{mailto:tristan.vaccon@unilim.fr}{tristan.vaccon@unilim.fr}} \\
       \addr 
       Universit\'{e} de Limoges\\
       Limoges, France 
       }

\begin{document}

\maketitle

\begin{abstract}
The intersection of deep learning and symbolic mathematics has seen rapid progress in recent years, exemplified by the work of Lample and Charton \cite{LC19}.
They demonstrated that effective training of machine learning models for solving mathematical problems critically depends on high-quality, domain-specific datasets.  In this paper, we address the computation of \gr\ basis using Transformers.
While a dataset generation method tailored to Transformer-based \gr\ basis computation has previously been proposed \cite{KIKVY24}, it lacked theoretical guarantees regarding the generality or quality of the generated datasets.
In this work, we prove that datasets generated by the previously proposed algorithm are sufficiently general, enabling one to ensure that Transformers can learn a sufficiently diverse range of Gröbner bases.
Moreover, we propose an extended and generalized algorithm to systematically construct datasets of ideal generators, further enhancing the training effectiveness of Transformer. 
Our results provide a rigorous geometric foundation for Transformers to address a mathematical problem, which is an answer to Lample and Charton's idea of training on diverse or representative inputs.

\end{abstract}

\section{Introduction}\label{sec:introduction}
In recent years, the practical success of machine learning models has sparked extensive research into their capabilities across diverse problem domains. This paper specifically examines the potential of Transformer-based models \cite{Transformer} for solving complex mathematical problems. 
Notably, previous research, such as on symbolic integration \cite{LC19}, has demonstrated that Transformers can outperform traditional mathematical software in terms of computational speed and efficiency.
Such findings underscore the potential value of investigating how machine learning models, particularly Transformers, perform on other mathematical challenges.
These promising results highlight the importance of exploring the effectiveness of Transformer-based models on other computationally demanding mathematical challenges, particularly those with substantial implications in cryptography. One significant area of interest is the Learning With Errors (LWE) problem \cite{LWEandCryp}, known to be computationally challenging in certain worst-case scenarios involving lattice problems. The complexity of the LWE problem underpins the security of contemporary lattice-based cryptographic schemes, such as CRYSTALS-Kyber \cite{CRYSTALS-Kyber3.02} and CRYSTALS-Dilithium \cite{Crystals-Dilithium3.1}, both recently standardized by the U.S. NIST as post-quantum cryptographic standards \cite{NIST2022}. Recent advancements, exemplified by the SALSA project \cite{SALSA,SALSA:VERDE,SALSA:PICANTE}, further illustrate the emerging capability of Transformers in addressing LWE-related challenges.
 There is another recent work by Alfarano, Charton, and Hayat \cite{ACH24}, which introduced a Transformer-based model capable of computing Lyapunov functions for dynamical systems, a longstanding open problem in mathematics. Lyapunov functions are essential for determining system stability, yet there has been no known deterministic algorithm to find them for general systems for over a century. Remarkably, their Transformers model rapidly identified Lyapunov functions for multiple previously unsolved systems, drastically reducing computation time from an average of 16 minutes (traditional numerical methods) to mere seconds. 
This emerging line of work highlights the growing potential of machine learning models to address mathematically complex and cryptographically relevant problems efficiently.
A common thread among these advancements is the potential of machine learning, particularly Transformers models, to address mathematical problems generally considered to be NP-hard or similarly complex. 

However, achieving this potential critically depends on the availability of high-quality datasets carefully designed to incorporate domain-specific knowledge tailored to each particular mathematical problem area.
In their paper, Lample and Charton \cite{LC19} emphasize the importance of generating training data that adequately spans the problem space.
They argue that a mixture of BWD and IBP leads to datasets that better represent the diversity of the problem space, implicitly aiming for a kind of “density” over the input distribution, albeit without a formal topological or algebraic framework.

Based on this background, we address problems in computational algebra using machine learning approaches, specifically focusing on the notion of \gr\ bases.                                                          
Given an ideal $I$ generated by a polynomial system $F$, a set $G = \langle g_1, g_2, \dots, g_t\rangle$ of generators of $I$ is called a \textit{\gr\ basis} with respect to a fixed monomial ordering if the following equation 
\[\LT(I)=\langle \LT(g_1),\LT(g_2 ),\cdots,\LT(g_t)\rangle\]
holds.
Here, $\LT(h)$ (resp. $\LT(I)$) denote the leading term of a polynomial $h$ (resp. the set of all leading terms of polynomials in $I$) and the brackets around $\cdot$ in $\langle \cdot \rangle$ denote the ideal generated by $\cdot$.           
Intuitively, a \gr\ basis provides a canonical form for representing the elements of $I$. A fundamental property of Gröbner bases is their ability to facilitate polynomial division, ensuring a unique remainder. This feature underpins their algorithmic effectiveness in numerous computational tasks within algebraic geometry, polynomial system solving, and related areas \cite{GTZ88,Eis95, CLO97,CLO98}.
Gröbner bases have significant applications, notably in simplifying and solving systems of multivariate polynomial equations \cite{CLO97}. Additionally, they are central to various cryptographic constructions. Similar to lattice-based problems, the computational difficulty associated with Gröbner bases serves as the security foundation for cryptographic primitives and protocols, including block ciphers \cite{BlockGrob}, stream ciphers \cite{AlgonStream}, and multivariate cryptographic schemes \cite{UOV,AlgHFE}. The applicability of Gröbner bases extends even further, impacting diverse fields such as inverse kinematics and path planning in robotics \cite{CLO97,YTM23}, control engineering \cite{AH00}, and biological system modeling \cite{systemsbiology}. Consequently, developing more efficient algorithms for Gröbner basis computation holds considerable promise for advancements across a wide spectrum of scientific and engineering domains.
The Gröbner basis computation problem, finding a Gröbner basis for a given polynomial ideal, is classically known to be NP-hard \cite{MM82}. Recent work \cite{KIKVY24} has explored machine learning approaches for tackling this challenge. Leveraging purely algebraic insights, they developed a method for training Transformers to compute Gröbner bases, demonstrating experimentally the potential for significant computational speed-ups. In several examples of polynomial systems, the model for computing Gröbner bases demonstrated a performance increase of up to 100 times compared to traditional methods, such as Buchberger's algorithm and the F4 algorithm \cite{Buchberger76,faugere1999f4}.
They devised an algorithm to generate datasets consisting of polynomial ideals paired with their corresponding Gröbner bases for model training. Considering that the quality of these datasets directly influences the performance of the learned models, it is essential to investigate their adequacy. However, the original method lacked theoretical guarantees regarding the generality or quality of the generated datasets.

In this paper, assuming a heuristic 
we prove that the dataset construction algorithm presented in \cite{KIKVY24} yields datasets with sufficient generality; see Problem \ref{prob:taking_generators} for the precise formulation. Furthermore, we introduce a novel algorithm designed to generate an even broader class of datasets.
In contrast to the empirical notion of representativeness \cite{LC19}, our work formalizes a method of dataset construction in which the generated examples are Zariski dense in the target space of Gröbner bases.
Zariski density, a concept from algebraic geometry, implies that the training data are not confined to a special subset but are instead topologically dense in the space of interest. In particular, our result ensures that any non-trivial algebraic property satisfied on the dataset extends to the full output space, and the trained model can, in principle, generalize to all generic inputs that define the same algebraic structure.
This complements the design for constructing training datasets of Lample and Charton by providing a rigorous geometric foundation for the idea of training on diverse or representative inputs.
Density is a geometrical concept that describes the generality of elements. A subset of a topological space is said to be dense if its closure coincides with the entire space. In the field of algebraic geometry, the generality of given elements is often assessed by examining the density of the subset comprising those elements. Following a conventional approach, we represent the set of polynomial systems as a union of affine spaces, where the coordinates correspond to the coefficients of the polynomial systems. The primary question addressed in this study is whether the dataset generated by the algorithm introduced in \cite{KIKVY24} is dense.

\section{Problems and main results}
\label{sec:problems}
This work aims to provide a theoretical explanation for the high learning accuracy observed in the Gröbner basis computation experiments over the field of rational numbers, as reported in \cite{KIKVY24}. In particular, we demonstrate that the training data generated by the algorithm satisfies a certain notion of density, which ensures a rich and diverse set of examples conducive to effective learning. Our result suggests that the empirical success in the rational setting is not incidental but supported by an underlying mathematical structure. While models perform less accurately over finite fields, possibly due to the absence of such density, our findings offer a structural perspective that can guide future theoretical investigations; see also Remark \ref{rem:coefficient fields}.
\subsection{A problem concerning datasets}
In the training of machine learning models for the computation of \gr\ bases, it is necessary to prepare a large-scale training dataset consisting of the pairs $(F,G)$, with $F=(f_1,\ldots,f_m)$ a generating set of an ideal that serves as an input of the model, and $G=(g_1,\ldots,g_n)$ a \gr\ basis of the ideal generated by $F$. Since \gr\ basis computation is NP-hard in general,
it is difficult to construct such a dataset on a large scale. In the algorithm \cite{KIKVY24} for computing the \gr\ basis, the output of Transformers is restricted to a certain class, called shape position \gr\ basis, parameterized freely by the coefficients of non-leading terms. Using shape position \gr\ bases, they constructed a large training dataset by creating polynomial systems that generate the same ideals as a given \gr\ basis (see section \ref{sec:algorithm} for more detail).

Let us consider an algorithm; for a given uniformly random set of \gr\ basis $\{G_1,\ldots,G_N\}$ belonging to a certain class $\mathcal{G}$ of \gr\ basis \footnote{For example, $\mathcal{G} = \{ G \mid \text{$G$ is a shape position \gr\ basis}\}$.}, it outputs a random training dataset 
\[\{(F_1,G_1),\ldots,(F_N,G_N)\}\]
following the uniform distribution over the set
\[ \{(F,G_i) \mid \text{$F$ is a polynomial system with}\ \langle F \rangle = \langle G_i \rangle\}.\]
For each \gr\ basis $G_i$, such an algorithm should be capable of outputting any pair $(F,G_i)$ such that $F$ could potentially be an input of the model. We connect this latter property of an algorithm to the notion of density in an algebraic variety.


For example, let $R = K[x_1,\ldots,x_r]$ be a polynomial ring over a field $K$.
Fix integers $m,n$ with $m \geq n=r$, noting that the number of elements in a shape position \gr\ basis is equals to the number of variables, and since the ideal $\langle G \rangle$ generated by $G$ is $0$-dimensional, the number of generators of $I$ is greater than or equal to the number of variables. The model considered in \cite{KIKVY24} took input-output pairs $(F,G)$, a shape position \gr\ basis $G = (g_1,\ldots,g_n)$, and a polynomial system $F=(f_1,\ldots,f_m)$ consisting of $m$ elements that generate the same ideal as $G$. The ideal algorithm for generating a generic training dataset, as envisioned in this work, should be capable of producing pairs $(F,G)$ such that $F$ is an arbitrary generic element of the set
\[ \mathcal{F} := \{ F = (f_1,\ldots, f_m) \in R^m \mid \langle G \rangle = \langle F \rangle \} \]
where $G$ is a fixed Gröbner basis in a specified shape position.

To construct a generic learning dataset, we focus on sets that are dense within a topological space. Informally, a subset $A$ of a topological space $X$ is dense if every point in $X$ lies arbitrarily close to some point in $A$. Thus, if a given learning dataset is dense in the topological space comprising all possible learning objects, we may conclude that the dataset is generic with respect to the given topology. In this study, we examine dense sets in $\mathcal{F}$ under a specific topology, namely the \textit{Zariski topology} \cite{Har77}.
                                                                                                                                               
More precisely, this paper aims to solve the following problem.

\begin{problem}\label{prob:taking_generators}
Provide an algorithm that, for a given set of polynomials $G=(g_1,\ldots,g_n)$ as input, randomly outputs an element $F$ of $\mathcal{F}$ such that the set of all possible outputs
\[\mathcal{F}_0 := \{F \in \mathcal{F} \mid \text{$F$ is an output of the algorithm}\} \]
is a dense subset of $\mathcal{F}$ for the Zariski topology.
\end{problem}

Here, we equip $R^m$ and $\mathcal{F}$ with the Zariski topology as follows; for a non-negative integer $D$, let $R_{\leq D} := \{ f \in R \mid \deg f \leq D\}$ be the subset consisting of elements in $R$ of total degree less than or equal to $ D$. The set $R_{\leq D}$ is a vector space over $K$ with a basis
\[ \{ x^{\alpha} = x_1^{\alpha_1} \cdots x_r^{\alpha_r} \mid |\alpha| = \alpha_1+\cdots+\alpha_r \leq D\}\]
of monomials in the $r$ variables. In general, a vector space over $K$ with a basis $\{e_1,\ldots,e_N\}$ can be identified with the affine space 
\[\mathbb{A}_K^N = \spec K[X_1,\ldots,X_N].\]
Then we identify $R_{\leq D}$ with the affine space  $K[X_{\alpha} \mid |\alpha| \leq D]$. In this paper, the Zariski topology of $R^m$ is defined as the canonical topology of the union of topological spaces 
\[ (R_{\leq 0})^m \subset (R_{\leq 1})^m \subset \cdots \subset (R_{\leq D})^m \subset \cdots \subset R^m = \bigcup_{D=0}^{\infty} (R_{\leq D})^m.\]
Namely, a subset $U$ of $R$ is open if and only if $U$ is the union of open subsets $U_D \subset (R_{\leq D})^m$. This identification on the set of polynomial systems is a traditional way in the area of research about the generality of properties of polynomial systems, for example, semi-regular sequence \cite{pardue2010generic}, and comprehensive \gr\ bases \cite{weispfenning1992comprehensive,suzuki2006simple}.

\subsection{Main results}
In this paper, we solve Problem \ref{prob:taking_generators} assuming a heuristic.
We propose an algorithm (Algorithm \ref{alg:main_algorithm_intro}) that randomly outputs a set of generators $F=(f_1,\ldots,f_m)$ of $\langle G \rangle$ for a given set of polynomials $G=(g_1,\ldots,g_n)$.

\begin{algorithm}
    \DontPrintSemicolon
	\SetKwInOut{Input}{input}\SetKwInOut{Output}{output}
     \caption{Algorithm to compute random generators of $\langle G \rangle$ (See Algorithm \ref{alg:main_algorithm} for a detailed version).}
    \label{alg:main_algorithm_intro}
    \Input{$G = (g_1,\ldots, g_n)^{T} \in R^{n \times 1}$, $m \geq n$.}
    \Output{$F = (f_1,\ldots,f_m)^{T} \in R^{m \times 1}$ such that $\langle F \rangle = \langle G \rangle$.}
    \BlankLine
    Pick a finite product $U$ of random elementary matrices; \;
	\Return $F = AG$ where
	\[ A = U\left(\begin{array}{c} E_n \\ \hline O_{(m-n) \times n} \end{array}\right). \] \;
\end{algorithm}

Here we denote by $E_n$ the identity matrix of size $n$, and by $O_{(m-n)\times n}$ the $(m-n) \times n$ zero matrix.
Let us recall
\begin{align*}
    \mathcal{F}& := \{ F \in R^m \mid \text{$\langle G \rangle$ is generated by $F$} \}, \\
    \mathcal{F}_0 &:= \{ F \in \mathcal{F} \mid \text{$F$ is an output of Algorithm \ref{alg:main_algorithm_intro}} \} \\
     &= \left\{ \left. U\left(\begin{array}{c} E_n \\ \hline O_{(m-n) \times n} \end{array}\right)G \right| \begin{aligned}
   \text{$U$ is a finite product of}\\
   \text{elementary matrices}
     \end{aligned}\right\}.
\end{align*}
Our heuristic assumption concerns the irreducibility of an algebraic set.

\begin{heuristic}\label{heu:chi_is_irr_intro}
Let us consider a block representation of $B \in R^{n \times m}$ and $A \in R^{m \times n}$:
\[ B = (B_1 | B_2),\ A = \left(\begin{array}{c} A_1 \\ \hline A_2 \end{array} \right),\]
where $B_1, A_1 \in R^{n \times n}$, $B_2 \in R^{n \times (m-n)}$ and $A_2 \in R^{(m-n) \times n}$. For given $D \geq 0$ and $G = (g_1,\ldots, g_n)$, the algebraic set
\[ \mathcal{X_{\leq D}} = \left\{ (B,A) \in R_{\leq D}^{n \times m} \times R_{\leq D}^{m \times n} \left| \begin{aligned} &(BA - E_n)G = O_{n \times 1},\\ &B_1 A_1 \in R_{\leq D}^{n \times n} \end{aligned} \right. \right\} \]
in the affine space $R_{\leq D}^{n \times m} \times R_{\leq D}^{m \times n}$ is irreducible.
\end{heuristic}

Here, our main theorem asserts the density in $\mathcal{F}$ of 
\[ \tilde{\mathcal{F}}_{\leq D} := \{ F \in R^{m \times 1} \mid \langle F \rangle = \langle G \rangle,\ \exists\, (B,A) \in \mathcal{X}_{\leq D},\ F = AG \}. \]

\begin{theorem}[{Corollary \ref{cor:m2n}}]
\label{thm:main}
Using the same notation as Problem \ref{prob:taking_generators}, assume that the base field $K$ is a Hilbertian field and $m \geq 2n \geq 3$. 
Then, under Heuristic \ref{heu:chi_is_irr_intro} for $D$ and $G$, the set $\mathcal{F}_0 \cap \tilde{\mathcal{F}}_{\leq D}$ is dense for the relative Zariski topology of $\tilde{\mathcal{F}}_{\leq D}.$
\end{theorem}

Here we say a field $K$ is a \textit{Hilbertian field} if Hilbert's irreducibility theorem (Theorem \ref{thm.Hilbert's irreducibility theorem}, \cite[Theorem 14.4.2]{FieldArithmetic2023}) holds for $K$. For example, number fields are Hilbertian.


\begin{corollary}
Assume that the same hypothesis in Theorem \ref{thm:main} holds. If Heuristic \ref{heu:chi_is_irr_intro} is true for sufficiently large integers $D \geq 0$ and a given $G$, then $\mathcal{F}_0$ is Zariski dense in $\mathcal{F}$.    
\end{corollary}

\begin{proof}
By Prop. \ref{prop:gene_iff} and the definition of $\mathcal{X}_{\leq D}$,  there exists an ascending chain of subsets
\[ \tilde{\mathcal{F}}_{\leq D} \subset \tilde{\mathcal{F}}_{\leq (D+1)} \subset \cdots \subset \mathcal{F} = \bigcup_{D' \geq D} \tilde{\mathcal{F}}_{\leq D'}. \]
Let $V$ be the Zariski closure of $\mathcal{F}_0$ in $\mathcal{F}$. It follows that $V \cap \tilde{\mathcal{F}}_{\leq D} = \tilde{\mathcal{F}}_{\leq D}$ for all sufficiently large $D \geq 0$ from the hypothesis and Theorem \ref{thm:main}. Therefore, we have
\[ V =V \cap \mathcal{F} =  V \cap \bigcup_{D' \geq D} \tilde{\mathcal{F}}_{\leq D'} =  \bigcup_{D' \geq D} V \cap \tilde{\mathcal{F}}_{\leq D'} =  \bigcup_{D' \geq D} \tilde{\mathcal{F}}_{\leq D'} = \mathcal{F}. \]
\end{proof}


The organization of this paper is as follows.
In Section \ref{sec:algorithm}, we provide elementary definitions and notations, and recall the method proposed in \cite{KIKVY24} for constructing the dataset, which we call the left regular matrix method. As a generalization, we introduce a new algorithm (Algorithm \ref{alg:main_algorithm}).
In Section \ref{sec:Density}, we study the density of the set 
\[ \mathcal{F}_0 = \left\{ \left. U \left(\begin{array}{c} E_n \\ \hline O_{(m-n) \times n} \end{array}\right)G \right| U \in \mathbb{E}(m) \right\}, \]
consisting of all possible outputs of Algorithm \ref{alg:main_algorithm}, in 
\[ \mathcal{F} =\{ F \in R^m \mid \langle F \rangle = \langle G \rangle \}.\]
First, we examine the algebraic aspects of $\mathcal{F}$ and $\mathcal{F}_0$ using the Quillen-Suslin theorem (Theorem \ref{thm:Serre_conj}). Subsequently, we show that $\mathcal{F}_0$ is Zariski dense in $\mathcal{F}$. Our main tool is Hilbert’s irreducibility theorem (Theorem \ref{thm.Hilbert's irreducibility theorem}), which asserts that the irreducibility of a generic irreducible polynomial is preserved, even after specialization by variable transformations under certain assumptions.

\begin{remark}
\label{rem:coefficient fields}
In the learning experiments on Gröbner basis computation conducted in \cite{KIKVY24}, an intriguing phenomenon was observed: the learning accuracy varied significantly depending on the choice of the coefficient field. When the coefficient field was the field of rational numbers, the model achieved an accuracy of approximately 90\%. In contrast, when the coefficient field was a finite field, the model's accuracy ranged from 50\% to 70\%.
In general, low learning accuracy in machine learning models can result from the intrinsic complexity of the task or the nature of the training data. Empirical studies have shown that Transformers exhibit difficulties when learning operations over finite fields \cite{gromov2023grokking, furuta2024towards, hahn2024sensitive}. This suggests that the challenge of learning coefficients in Gröbner basis computation may arise from the same underlying and yet-to-be-understood mathematical mechanisms that govern computation over finite fields.
On the other hand, when the coefficients are drawn from the field of rational numbers, it is plausible that the task of computing Gröbner bases becomes relatively more tractable, and that the training data covers a broader and more expressive range of instances, thereby enabling higher learning accuracy.
In our study, we have theoretically demonstrated a certain notion of density in the construction of training datasets generated by the algorithm in \cite{KIKVY24}. This density result implies that, over the field of rational numbers, the learning data exhibits sufficient richness and diversity to support high-accuracy learning. Conversely, over finite fields, such density may not hold—or may be inherently limited—due to structural constraints, potentially explaining the lower empirical accuracy observed. In this sense, our result serves as a kind of theoretical justification for the empirical performance gap between the rational and finite field settings, suggesting that the success over the rationals is not incidental but backed by underlying mathematical guarantees.
\end{remark}

\begin{remark}\label{rem:range_of_n}
In the experiments of \cite{KIKVY24}, a constraint of $m \leq n + 2$ is imposed between the number of elements $n$ of $G$ and the number of elements $m$ of $F$, due to the input length limitations of the vanilla Transformer architecture. Since the main theorem of this paper assumes $m \geq 2n$, the case $n = 2$ is the only setting, while they conjecture that high accuracy can still be achieved over the field of rational numbers even when $m > n + 2$, and the main theorem of this paper provides theoretical support for this belief.
\end{remark}

\subsection*{Acknowledgment}
The authors would like to express their gratitude to Yuki Ishihara, Kazuhiro Yokoyama, and Hiroshi Kera for their many comments and advices regarding the early version of this paper.
This work was also partially supported by the Alexander von Humboldt Foundation through a Humboldt-Research fellowship granted to the second author.

\section*{Notation and convention}

\begin{itemize}
	\item Let $K$ be a field, and $r,m,n$ be integers $\geq 1$ with $m \geq n.$
	\item Let $R = K[x_1,\ldots, x_r]$ be the polynomial ring in $r$ variables over $K.$ For a vector of non-negative integers $\alpha = (\alpha_1,\ldots,\alpha_r) \in \mathbb{Z}_{\geq 0}^r,$ we denote by $x^{\alpha}$ the monomial
 \[ x^{\alpha} := x_1^{\alpha_1} \dots x_r^{\alpha_r} \]
 of total degree $\vert x^\alpha \vert = \sum_{i=1}^r \alpha_i$ using multi-index. For a polynomial $f \in R,$ we define its total degree, denoted by $\deg f,$ to be the maximum among the total degrees of monomials appearing in $f.$
    \item As stated in the Introduction, let
    \[R_{\leq D} := \{ f  \in R \mid \deg f \leq D \} = \left\{\left. \sum_{|\alpha| \leq D} c_{\alpha} x^{\alpha} \right| \{c_{\alpha}\}_{|\alpha|\leq D} \subset K \right\}.\]
    We denote by $N_D$ the number of monomials in $R_{\leq D}.$ Note that we have $N_D = \sum_{d = 0}^D \binom{r+d}{d}$. 

    \item For a tuple of polynomials $F = (f_1,\ldots,f_m) \in R^m$, we denote by $\langle F \rangle = \langle f_1,\ldots,f_m \rangle$ the ideal generated by $f_1,\ldots,f_m$.
    

 
\end{itemize}

In this paper, we do not recall the elementary notions of \gr\ basis theory since we will not assume that the given polynomial set $G$ is a (shape position) \gr\ basis in our results. The interested reader may refer to \cite{CLO97} for an introduction to \gr\ basis theory.



\section{Algorithm for constructing a training dataset}\label{sec:algorithm}

\subsection{Left regular matrix methods}
To obtain a large-scale set of \gr\ basis, in \cite{KIKVY24}, the authors restricted to the case of shape position \gr\ bases. This is the generic case for \gr\ bases of zero-dimensional radical ideals.

\begin{proposition}[{\cite[Proposition 1.6]{ShapeLemma89}}]\label{prop:shape_position}
Let $I$ be a 0-dimensional radical ideal in $K[x_1,\ldots,x_n]$ with the lexicographic order $x_1 > \cdots > x_n$. If $K$ is of characteristic 0 or a finite field having large enough cardinality,
then a random linear coordinate change puts $I$ in shape position. In other words, the reduced \gr\ basis of $I$ consists of $g_1,\cdots ,g_n$ so that 
\[ \begin{aligned}
    g_1 &= x_1 - h_1(x_n),\\
    & \ \vdots \\
    g_{n-1} & = x_{n-1} - h_{n-1}(x_n),\\
    g_n & = g_n(x_n),
\end{aligned}\]
where $h_1,h_2,\ldots,h_{n-1}$ and $g_n$ are polynomials in $x_n$. 
\end{proposition}
We call a \gr\ basis in the above form a \textit{shape position \gr\ basis}.
The set of all shape position \gr\ basss of degree $D$ is parameterized by the affine space $(\mathbb{A}^{D}_K)^{n}$ (parameterizing the coefficients of the non-leading terms of the $g_i$'s). Thus we can easily construct a large dataset of shape position \gr\ bases following a given distribution.

Throughout this paper, we are supposed to have an algorithm for producing a large dataset $\mathcal{G}$ consisting of \gr\ bases contained in a class of the target of the model. It follows that we only focus on algorithms for constructing sets of generators $F$ such that $\langle F \rangle = \langle G \rangle$ for a given element $G \in \mathcal{G}$.
The basic idea in \cite{KIKVY24} for constructing such an algorithm is the following.
Let  \[ G = (g_1,\ldots,g_n)^{T}=\begin{pmatrix} g_1\\
\vdots \\
g_n \end{pmatrix} \in R^{n \times 1} \]
be a tuple of polynomials.
If another tuple of polynomials
\[ F = (f_1,\ldots, f_m)^{T} = \begin{pmatrix} f_1\\
\vdots \\
f_m \end{pmatrix}  \in R^{m \times 1} \]
satisfies $f_i \in \langle G \rangle$ for any $i$, 
then there exists a polynomial matrix $A \in R^{m \times n}$ with
\[ F = AG. \]
In this setting, if $A$ has a left inverse matrix $B \in R^{n \times m}$, then $F$ is a set of generators of $\langle G \rangle$ since $BF = BAG = G$. Let us work on such matrices.

\begin{definition}\label{def:left_regular}
We say a matrix $A \in R^{m \times n}$ is a \textit{left regular matrix} over $R$ if there exists a left inverse matrix $B \in R^{n \times m}$. We simply say $A$ is \textit{regular} over $R$ if $A$ is left regular and $m=n$.
\end{definition}

To construct a left regular matrix $A \in R^{m \times n}$ over $R$ efficiently, one can consider the Bruhat-like decomposition
\[ A = U_1 S \left(\begin{array}{c} U_2 \\ \hline O_{(m-n) \times n} \end{array}\right).\]
Here, $U_1 \in R^{m \times m}$ and $U_2 \in R^{n \times n}$ are upper triangle matrices with diagonal entries all 1, $O_{(m-n)\times n}$ is the zero matrix of size $(m-n) \times n$, and $S$ is a permutation square matrix of size $m$.
Clearly $A$ has a left inverse matrix
\[ B = \left(U_2^{-1} | O_{n \times (m-n)} \right)S^{-1} U_1^{-1}.\]
This matrix $B$ is in $R^{n \times m}$ since $\det U_1 = \det U_2 = \det S =  1$.
Then the matrix $A$ is left regular over $R$.

In this paper, we first consider the generalization of the above.
If $V_1 \in R^{m \times m}$ and $V_2 \in R^{n \times n}$ are regular matrices over $R$, then
\[ A = V_1 \left(\begin{array}{c} V_2 \\ \hline O_{(m-n) \times n} \end{array}\right) \]
is a left regular matrix over $R$. Moreover, from the transformation
\[ A = V_1 \left(\begin{array}{c|c} V_2 & O_{n \times (m-n)} \\ \hline O_{(m-n) \times n} & E_{(m-n) \times (m-n)} \end{array}\right) \left(\begin{array}{c} E_n \\ \hline O_{(m-n) \times n} \end{array}\right), \]
it is enough to consider the case of $V_2 = E_n$, the identity matrix of size $m\times m$. Hence, the problem of constructing left regular matrices $A \in R^{m \times n}$ is reduced to find regular matrices $V \in R^{m \times m}$. Note that we will show that a polynomial matrix $A\in R^{m \times n}$ is a left regular matrix over $R$ if and only if $A$ is in the above form (Proposition \ref{prop:set_of_outputs}).
Now, let us clarify the set of all regular matrices over $R$ of size $m$.

\begin{proposition}\label{prop:regular_det}
A square matrix $V \in R^{m \times m}$ is a regular matrix over $R$ if and only if $\det V \in K \setminus \{0\}$. 
\end{proposition}

\begin{proof}
If $V$ is regular, then there exists a polynomial matrix $U \in R^{m \times m}$ such that $UV = E_m$. Then the determinant $\det V$ is a unit of $R$. Since the set of units of $R$ is the set $K \setminus \{0\}$, the determinant $\det V$ is in $K \setminus \{0\}$. Conversely, if $\det V \in K \setminus \{0\}$, then the inverse $V^{-1} = (\det V)^{-1} \tilde{V}$, where $\tilde{V} \in R^{m \times m}$ is the adjugate matrix of $V$, is in $R^{m \times m}$. It implies that $V$ is regular over $R$.
\end{proof}

Let $V \in R^{m \times m}$ be a regular matrix of size $m$. Consider the matrix
\[ U = ((\det V)^{-1}\bm{e}_1,\bm{e}_2,\ldots,\bm{e}_m) V,\]
where $\bm{e}_i$ is the $i$-th elementary vector in $R^m$. 
It is an element of the special linear group
\[ \mathrm{SL}(R^m) := \{ U \in R^{m \times m} \mid \det U = 1 \}\]
over $R$. It follows that 
\[ \{ V \in R^{m \times m} \mid \text{$V$ is regular over $R$}\} = \langle \mathrm{GL}(K^m), \mathrm{SL}(R^m) \rangle \]
 in $R^{m \times m}$, where
\[ \mathrm{GL}(K^m) := \{ C \in K^{m \times m} \mid \det C \neq 0 \}\]
is the general linear group over $K$, and $\langle \mathrm{GL}(K^m), \mathrm{SL}(R^m) \rangle$ is the subgroup generated by $\mathrm{GL}(K^m)$ and $\mathrm{SL}(R^m)$ in $R^{m \times m}.$
To describe this subgroup in detail, we recall Suslin's stability theorem claiming that elementary matrices generate it.

\begin{definition}\label{def:elementary_matrix}
Let $P$ be a $m\times m$ matrix over $R$. 
We say $P$ is an \textit{elementary} matrix of size $m$ if it has one of the following forms.
\begin{itemize}
	\item A row permutation matrix:
 \[ P = (\bm{e}_1,\ldots,\bm{e}_j,\ldots,\bm{e}_i,\ldots,\bm{e}_m)
\]
for some $i<j$.
	\item A row multiplication matrix for an element of $K \setminus \{0\}$:
\[ P = (\bm{e}_1,\ldots, c \bm{e}_i, \ldots, \bm{e}_m)\]
for some $c \in K \setminus \{0\}$ and $i = 1,\ldots,m$.
	\item A row addition matrix for an element of $R$:
 \[ P = (\bm{e}_1,\ldots, \bm{e}_i+f\bm{e}_j, \ldots, \bm{e}_m)\]
 for some $f \in R$ and $i,j = 1,\ldots, m$ such that $i\neq j$.
\end{itemize}
We denote by $\mathbb{E}(m)$ the subgroup of $R^{m \times m}$ generated by all elementary matrices of size $m$.
\end{definition}

\begin{theorem}[Suslin's stability theorem  \cite{Suslin77}]\label{thm:Suslin's_stability_theorem}
If $m \geq 3$, then
\[ \mathbb{E}(m) = \langle \mathrm{GL}(K^m),\mathrm{SL}(R^m) \rangle.\]
\end{theorem}


Note that, of course, even in the case of $m < 3$, we have
\[ \mathbb{E}(m) \subset \langle \mathrm{GL}(K^m),\mathrm{SL}(R^m) \rangle. \]
How many matrices of $\mathbb{E}(m)$ are enough
to obtain a given element of $\langle \mathrm{GL}(K^m),\mathrm{SL}(R^m) \rangle$ as their product has been
studied in \cite{Caniglia93} with explicit bounds.

These observations lead us to a new algorithm to construct random sets of generators of the given ideal $\langle G \rangle$ with Algorithm \ref{alg:main_algorithm}. 
When picking at random in the algorithm, we mean according to any
reasonable distribution (\textit{e.g.} non-atomic, uniform $\dots$).

\begin{algorithm}
    \DontPrintSemicolon
	\SetKwInOut{Input}{input}\SetKwInOut{Output}{output}
     \caption{Algorithm to construct random generators of $\langle G \rangle$}
    \label{alg:main_algorithm}
    \Input{$G = (g_1,\ldots, g_n)^{T} \in R^{n \times 1},$ $m \geq n.$}
    \Output{$F = (f_1,\ldots,f_m)^{T} \in R^{m \times 1}$ such that $\langle F \rangle = \langle G \rangle.$}
    \BlankLine
    Pick a random integer $s \geq 1$ ; \;
    Pick random elementary matrices $U_1,U_2,\ldots,U_s$ of size $m$ ; \;
    Compute the product $U = U_1 U_2 \cdots U_s$;\;
    Compute the matrix
    \[ A = U \left(\begin{array}{c} E_n \\ \hline O_{(m-n) \times n} \end{array}\right); \]  \;
    \Return $F = AG$ \;
\end{algorithm}


\begin{proposition}\label{prop:main_algorithm}
Algorithm \ref{alg:main_algorithm} is correct. In other words, For any input $G$, the output $F$ generates the ideal  $\langle G\rangle$.
\end{proposition}

\begin{proof}
Since a product of any elementary matrix is regular over $R$, the matrix $U$ in step 3 is regular over $R$. Then $A$ is also left regular over $R$ since it holds that
\[ B = (E_n | O_{n\times(m-n)})U^{-1} \in R^{n \times m}\]
and $BA = E_n$, which forces $\langle F \rangle = \langle G \rangle$.
\end{proof}

\section{Density of the outputs}\label{sec:Density}

The goal of this section is to show that 
\[ \mathcal{F}_0 = \left\{ \left. U \left(\begin{array}{c} E_n \\ \hline O_{(m-n) \times n} \end{array}\right)G \right| U \in \mathbb{E}(m) \right\} \]
is dense in
\[ \mathcal{F} =\{ F \in R^m \mid \langle F \rangle = \langle G \rangle \}.\]

\subsection{Algebraic aspects of \texorpdfstring{$\mathcal{F}$}{F} and  \texorpdfstring{$\mathcal{F}_0$}{F0}}

First, we show Proposition \ref{prop:set_of_outputs}, asserting that $\mathcal{F}_0$ is the same as the set of generators $F = AG$ whose $A$ is a left regular matrix over $R$.

\begin{theorem}[Quillen-Suslin theorem, also known as Serre's conjecture]\label{thm:Serre_conj}
A finitely generated projective $R$-module is free.
\end{theorem}
\begin{proof}
    See \cite[XXI Theo. 3.7]{Lang:Algebra}
\end{proof}

\begin{proposition}\label{prop:set_of_outputs}
For any left regular matrix $A \in R^{m \times n}$ with $m \geq n \geq 3$, there exists an element $U \in \mathbb{E}(m)$ such that
\[ A = U \left(\begin{array}{c} E_n \\ \hline O_{(m-n) \times n} \end{array}\right). \]
\end{proposition}

\begin{proof}
Let $B \in R^{n \times m}$ be a matrix such that $BA = E_n$. Let $\varphi_A : R^n \rightarrow R^m$ and $\varphi_B : R^m \rightarrow R^n$ be the $R$-module morphisms defined by  $\varphi_A(C) = AC$ and $\varphi_B(C) = BC.$ The assumption on $B$ implies that the sequence
\begin{equation}\label{eq:sequence}
    0 \rightarrow \mathrm{Ker}(\varphi_B) \rightarrow R^m \overset{\varphi_B}{\rightarrow} R^n \rightarrow 0
\end{equation}
is a split exact sequence. The kernel $\mathrm{Ker}(\varphi_B)$ is projective. Indeed, this is a direct summand of the free module $R^m$. Then, by Theorem \ref{thm:Serre_conj}, $\mathrm{Ker}(\mathrm{\varphi_B})$ is free. Moreover, since $R^n$ is projective over $R$, we have
\[ \mathrm{Ext}^1_R(R^n, \mathrm{Ker}(\varphi_B))=0.\]
It follows that $\mathrm{Ker}(\varphi_B)$ is isomorphic to the free module $R^{(m-n)}$ and the sequence \eqref{eq:sequence} corresponds to the zero of the module $\mathrm{Ext}^1_R(R^n, R^{(m-n)})$, which forces that 
\[ 0 \rightarrow R^{(m-n)} \rightarrow R^n \stackrel{\pr_n}{\to} R^m \rightarrow 0, \]
where $\pr_n$ is the projection to the first $n$ coordinates. Hence there exist isomorphisms $\rho : R^m \rightarrow R^m$ and $\tau : R^n \rightarrow R^n$ of $R$-modules making the diagram
\[
\xymatrix{
R^m \ar[r]^{\varphi_B} \ar[d]_{\rho} & R^n \ar[d]_{\tau} \\
R^m \ar[r]^{\pr_n} & R^n
}
\]
commutative. Let $U$, $V$, $U'$, $V'$ be matrices that represent $\rho$, $\tau$, $\rho^{-1}$, $\tau^{-1}$ respectively. These are matrices with entries in $R$ since $\rho^{-1}$ and $\tau^{-1}$ are still $R$-module morphisms. Then we have $\det(U), \det(V) \in R^{\times} = K \setminus \{0\}$ from Proposition \ref{prop:regular_det}. Therefore we obtain
\[ B = V^{-1} (E_n|O_{n \times (m-n)}) U\quad (V \in \mathbb{E}(n),\ U \in \mathbb{E}(m)) \]
from Theorem \ref{thm:Suslin's_stability_theorem}. Taking a block matrix representation
\[ UA = \left(\begin{array}{c} A_1 \\ \hline A_2 \end{array} \right) \]
for $A_1 \in R^{n \times n}$ and $A_2 \in R^{(m-n) \times n}$, we have
\[ E_n = BA = BU^{-1} UA = V^{-1}A_1. \]
It implies that
\[ A = U^{-1} \left( \begin{array}{c} V \\ \hline A_2 \end{array} \right) = U^{-1} \left( \begin{array}{c|c} V & O_{n \times (m-n)} \\ \hline A_2 & E_{(m-n)\times (m-n)}\end{array} \right)\left( \begin{array}{c} E_n \\ \hline O_{(m-n) \times n } \end{array} \right).\]
By the row reduction, we can show that the matrix 
\[U^{-1} \left( \begin{array}{c|c} V & O_{n \times (m-n)} \\ \hline A_2 & E_{(m-n)\times (m-n)}\end{array} \right)\]
is an element in $\mathbb{E}(m)$.
\end{proof}

From Proposition \ref{prop:set_of_outputs}, we obtain another representation of the subset $\mathcal{F}_0$ of the all possible outputs of Algorithm \ref{alg:main_algorithm} as
\[ \begin{aligned} \mathcal{F}_0 &= \left\{ \left. U \left(\begin{array}{c} E_n \\ \hline O_{(m-n) \times n} \end{array}\right)G \right| U \in \mathbb{E}(m) \right\}\\
&= \{ F=AG \in \mathcal{F} \mid A \in R^{m \times n}\ \text{is left regular over $R$}\}. \end{aligned} \]
We will use this representation to construct a dense subset $\mathcal{V} \subset \mathcal{F}_0$ consisting of elements of the form $F =AG$ with left regular matrices $A$ (Proposition \ref{prop:enough_condition_of_output}, Theorem \ref{thm:density_of_outputs}).
Next, we analyze $\mathcal{F} = \{ F \in R^{m \times 1}\mid \langle F \rangle = \langle G \rangle \}$.

\begin{proposition}\label{prop:gene_iff}
Let $F = AG$ for $A \in R^{m \times n}$. The tuple of polynomials $F$ is a generator of $\langle G \rangle$ if and only if there exists a matrix $B \in R^{n \times m}$ such that $(BA-E_n)G = O_{n \times 1}$.
\end{proposition}

\begin{proof}
 The polynomials $F=AG$ is a set of generators of $\langle G\rangle$ if and only if 
 there exists a matrix $B \in R^{n \times m}$ such that $G = BF= BAG$. The latter is rephrased as  $(BA-E_n)G = O_{n \times 1}$.
\end{proof}

\begin{definition}
    We call a polynomial row vector $H = (h_1,\ldots,h_n) \in R^{1 \times n}$ a \textit{syzygy} of $G$ if $HG = h_1g_1 + \cdots + h_n g_n = 0$. The set of all syzygies of $G$ is denoted by $\Syz(G) \subset R^{1 \times n}$. We also denote by $\Syz(G)^m \subset R^{m \times n}$ the set consisting of matrices that all rows are a syzygy of $G$.
\end{definition}
 Namely,
\[ \Syz(G)^m = \{ C \in R^{m \times n} \mid CG = O_{n \times 1} \}. \]
Using Proposition \ref{prop:gene_iff} and the syzygies of $G$, we will find enough condition for $F \in \mathcal{F}$ to be in $F \in \mathcal{F}_0$.

\begin{proposition}\label{prop:irreducible_determinant}
Assume that $m \geq n \geq 3$. Let us denote by 
\[ B = (B_1 | B_2), A = \left(\begin{array}{c} A_1 \\ \hline A_2 \end{array} \right)\quad (B_1, A_1 \in R^{n \times n}, B_2 \in R^{n \times (m-n)}, A_2 \in R^{(m-n) \times n}) \]
a block representations of $B \in R^{n \times m}$ and $A \in R^{m \times n}$. Let $F = AG \in \mathcal{F}$ be a set of generators of $\langle G \rangle$ and $B$ a matrix such that $W := BA - E_n \in \Syz(G)^n$. If the determinant of $E_n + W - B_2A_2$ is a non-zero irreducible polynomial in $R$, then $F \in \mathcal{F}_0$.
\end{proposition}

\begin{proof}
We note that
\[ E_n + W -B_2 A_2 = B_1A_1. \]
The assumption on the determinant of $E_n + W - B_2A_2$ implies that  $\det B_1 \in K \setminus \{0\}$ or $\det A_1 \in K \setminus \{0\}.$

First suppose that $\det B_1 \in K \setminus \{0\}$. In this case, we have
\[ A_1 = B_1^{-1}(E_n + W - B_2A_2).\]
Let us consider a matrix
\[ A' := \left( \begin{array}{c} B_1^{-1}(E_n - B_2 A_2) \\ \hline A_2 \end{array} \right) \in R^{m \times n}. \]
Then we have
\[ BA' = B_1 (B_1^{-1}(E_n - B_2A_2)) + B_2 A_2 = E_n\]
and
\[ AG = \left( \begin{array}{c} B_1^{-1}(E_n+W-B_2 A_2)G \\ \hline A_2G \end{array} \right) = \left( \begin{array}{c} B_1^{-1}(E_n - B_2 A_2)G \\ \hline A_2G \end{array} \right) = A'G\]
since $W \in \Syz(G)^n$. Hence the set of generators $F = AG = A'G$ is an element of $\mathcal{F}_0$ by Proposition \ref{prop:set_of_outputs}.

For the case of $\det A_1 \in K \setminus \{0\}$. putting 
\[ B' := (A_1^{-1} \mid O_{n\times (m-n)}) \in R^{n \times m}, \]
 we have
\[ B' A = (A_1^{-1} \mid O_{n\times (m-n)} )\left(\begin{array}{c} A_1 \\ \hline A_2 \end{array} \right) =  E_n. \]
Proposition \ref{prop:set_of_outputs} again concludes the claim.
\end{proof}

\subsection{Geometric aspects of \texorpdfstring{$\mathcal{F}$}{F} and \texorpdfstring{$\mathcal{F}_0$}{F0}}

Let us denote by
\[ \mathcal{X} := \{ (B,A) \mid B \in R^{n \times m}, A \in R^{m \times n}, (BA-E_n)G = O_{n \times 1} \} \]
an algebraic set in the affine space $R^{n\times m} \times R^{m \times n}$. Proposition \ref{prop:gene_iff} claims that there exists a surjective continuous map
\[ \varphi : \mathcal{X} \rightarrow \mathcal{F} \]
\[ (B,A) \mapsto AG. \]

The goal of this subsection is to find a dense subset $\mathcal{X}_0$ of $\mathcal{X}$ such that $\varphi(\mathcal{X}_0) \subset \mathcal{F}_0.$ 
The existence of such a subset implies that $\mathcal{F}_0$ is dense in $\mathcal{F}$ from the following basic topological fact that the image of a dense subset of a topological space is also dense under a continuous subjective map.



Let us consider a block representation of $B \in R^{n \times m}$ and $A \in R^{m \times n}$:
\[ B = (B_1 | B_2),\ A = \left(\begin{array}{c} A_1 \\ \hline A_2 \end{array} \right),\]
where $B_1, A_1 \in R^{n \times n}$, $B_2 \in R^{n \times (m-n)}$ and $A_2 \in R^{(m-n) \times n}$. 
According to these, we shall define
\[ p: \mathcal{X} \rightarrow R^{n \times n} \]
\[ (B,A) \mapsto B_1 A_1\]
and
\[ \delta : R^{n \times n} \rightarrow R\]
\[ C \mapsto \det(C). \]

The following proposition is another version of Proposition \ref{prop:irreducible_determinant} in the geometrical context.
Put
\[ \mathcal{U} := \{ f \in R \setminus \{0\} \mid \text{$f$ is irreducible} \}.\]

\begin{proposition}\label{prop:enough_condition_of_output}
Assume that $m \geq n \geq 3$. 
Then the following inclusion of subsets in $\mathcal{X}$ is true:
\[ (\delta \circ p)^{-1}(\mathcal{U}) \subset \varphi^{-1}(\mathcal{F}_0). \]
\end{proposition}

\begin{proof}
Pick an element $(B,A) \in (\delta \circ p)^{-1}(\mathcal{U}) \subset \mathcal{X}$. The image is
\[ (\delta \circ p)(B,A) = \det(B_1A_1).\]
Then similar to the proof of Proposition \ref{prop:irreducible_determinant}, $(\delta \circ p)(B,A) \in \mathcal{U}$ implies that $AG = \varphi(B,A) \in \mathcal{F}_0$.
\end{proof}

From a well-known fact, which is so-called Hilbert's irreducibility theorem, the inverse image
\[ \delta^{-1}(\mathcal{U}) \cap R_{\leq D}^{n \times n} = \{ C \in R_{\leq D}^{n \times n} \mid \text{$\det (C)$ is non-zero irreducible} \} \]
is Zariski dense in $R_{\leq D}^{n \times n} = (\mathbb{A}_K^{N_D})^{n \times n}$ over $K$.

\begin{theorem}[Hilbert's irreducibility theorem]\label{thm.Hilbert's irreducibility theorem}
Let $f(X_1,\ldots,X_r,Y_1,\ldots,Y_s)$ be an irreducible element in $\mathbb{Q}(X_1,\ldots,X_r)[Y_1,\ldots,Y_s]$. Then the set
\[ \{ (a_1,\ldots,a_r) \in \mathbb{Q}^r \mid f(a_1,\ldots,a_r)(Y_1,\ldots,Y_s)\ \text{is irreducible} \}\]
is Zariski dense in $\mathbb{A}^r_{\mathbb{Q}}$.
\end{theorem}
\begin{proof}
    See \cite[Theorem 14.4.2]{FieldArithmetic2023}.
\end{proof}

As written in section \ref{sec:introduction}, we say a field $K$ is Hilbertian if Hilbert's irreducibility theorem holds when replacing $\mathbb{Q}$ to $K$. 
To apply Hilbert's irreducibility theorem to our situation, let us prove the following lemma on the irreducibility of the determinant of a generic polynomial matrix.

\begin{lemma}\label{lem:irreducibility_of_determinant}
The polynomial
\[ \mathrm{DET}(\{c_{\alpha,i,j}\},\{x_k\}) \in K[c_{\alpha,i,j} \mid x^{\alpha} \in R_{\leq D}, 1 \leq i,j \leq n][x_1,\ldots,x_n]\]
corresponding to the map
\begin{align*}
    \delta: R_{\leq D}^{n \times n} &\rightarrow R \\
    C = \left(\sum_{|\alpha|\leq D} c_{\alpha,i,j}x^{\alpha} \right) &\mapsto \det C = \mathrm{DET}(\{c_{\alpha,i,j}\},\{x_k\})
\end{align*}
is irreducible in $K(c_{\alpha,i,j} \mid x^{\alpha} \in R_{\leq D}, 1 \leq i,j \leq n)[x_1,\ldots,x_n]$.
\end{lemma}
\begin{proof}
    Let 
    \[\Gamma := \mathrm{DET}(\{c_{\alpha,i,j}\},\{x_k\}) \in K[c_{\alpha,i,j} \mid x^{\alpha} \in R_{\leq D}, 1 \leq i,j \leq n][x_1,\ldots,x_n].\]
    We first prove that $\Gamma$ is irreducible in \[K(x_1,\ldots,x_n)[c_{\alpha,i,j}\  (x^{\alpha} \in R_{\leq D}, 1 \leq i,j \leq n)].\]



The proof can be understood as an extension of the classical proof of irreducibility
of the determinant of a matrix of generic entries, see \textit{e.g.} \cite[p.176]{Bocher:HigherAlgebra}.

We recall
Leibniz's formula for determinants:
\begin{equation}
    \Gamma = \sum_{\sigma \in \mathfrak{S}_n} \textrm{sign}(\sigma) \prod_{l=1}^n \left(\sum_{|\alpha|\leq D} c_{\alpha,l,\sigma(l)}x^{\alpha} \right). \label{eqn:Leibniz_formula}
\end{equation} 
Suppose $\Gamma = fg$ for some polynomials
\[f,g \in K[c_{\alpha,i,j} \: (x^{\alpha} \in R_{\leq D}, 1 \leq i,j \leq n), \: x_1,\ldots,x_n].\]
By \eqref{eqn:Leibniz_formula}, $\Gamma$ has degree at
most 1 with respect to any of the $c_{\alpha,i,j}.$
For any $\alpha,i,j$, adequate substitutions of $0$'s or $1$'s
prove that $\Gamma$ can be evaluated to 
\[\det \left( 
\begin{bmatrix} 1 & 0      &  &   & & \cdots & & & & & 0\\
                0 & \ddots &        &   & & & & & & & \\
                 &    & 1 & & & & & 0& & & \\
                        &   &   &0 & & & 0& 1&0 & & \\
                        &   &   &  & 1 & & &0 & & & \\
                \vdots        &   &   &   &   & \ddots & & & & & \vdots \\
                        &   &   & 0  &   &       & 1 & & & & \\
                        &   & 0  &c_{\alpha,i,j}   & 0  &       &  & 0 & & & \\
                        &   &   & 0  &   &       &  &  & 1& & \\
                        &   &   &   &   &       &  &  & & \ddots &0 \\
               0         &   &   &   &   &   \cdots    &  &  & & 0 & 1 \\ 
                        \end{bmatrix} \right) = \pm c_{\alpha,i,j}.\]
Consequently, all the variables $c_{\alpha,i,j}$ must appear in $\Gamma$.

Suppose that for some $\alpha$, $c_{\alpha,1,1}$
appears in $f.$
By consideration of degrees in $c_{\alpha,1,1}$, it then can not appear in $g.$
Leibniz's formula \eqref{eqn:Leibniz_formula} proves that there is no term where
$c_{\alpha,1,1}c_{\alpha,1,j}$ can appear in $\Gamma$ with $j \neq 1.$
It follows from the equality 
\[(u_1c_{\alpha,1,1}+u_2)(u_3c_{\alpha,1,j}+u_3)=u_1u_2c_{\alpha,1,1}c_{\alpha,1,j}+u_1u_3c_{\alpha,1,1}+u_2u_3c_{\alpha,1,j}+u_2u_3,\]
 that $c_{\alpha,1,j}$ can not appear in $g$
and must appear in $f$.
Similarly, all the $c_{\alpha,i,1}$ must appear in $f$ and not in $g$.
With the same argument applied for $c_{\alpha,i,1}$,  we can prove that all $c_{\alpha,i,j}$'s must appear in $f$ and not in $g$.

If we take some $\beta \neq \alpha$ with $\vert \beta \vert \leq D,$
\eqref{eqn:Leibniz_formula} proves that there is no term in 
$c_{\alpha,1,1} c_{\beta,1,1}$ in $\Gamma.$
Thus, $c_{\beta,1,1}$ must appear in $f$ and not in $g$.
Using the previous argument, we gets that all the
$c_{\beta,i,j}$ must appear in $f$ and not in $g.$
Therefore, $g$ can not contain any $c_{\alpha,i,j}$
for $\vert \alpha \vert \leq D, 1\leq i,j \leq n.$
Thus, $g$ can only be a pure polynomial in $K[x_1,\dots,x_n]$
and $\Gamma$ is then irreducible in 
$K(x_1,\ldots,x_n)[c_{\alpha,i,j} \: (x^{\alpha} \in R_{\leq D}, 1 \leq i,j \leq n)]$.

As a polynomial in $K[x_1,\ldots,x_n][c_{\alpha,i,j} \: (x^{\alpha} \in R_{\leq D}, 1 \leq i,j \leq n)]$, its content
has to be $1.$ Indeed, from \eqref{eqn:Leibniz_formula}, the coefficient of the monomial 
$c_{0,1,1}\cdots c_{0,n,n}$ in $\Gamma$ is $1$. 
By Gauss's lemma, it implies that $\Gamma$ is irreducible in
$K[x_1,\ldots,x_n][c_{\alpha,i,j} \: (x^{\alpha} \in R_{\leq D}, 1 \leq i,j \leq n)]$, and then, again by Gauss's lemma, $\Gamma$ is irreducible in 
    $K(c_{\alpha,i,j} \mid x^{\alpha} \in R_{\leq D}, 1 \leq i,j \leq n)[x_1,\ldots,x_n]$.    
\end{proof}


\begin{lemma}\label{lem:U_is_dense}
The set $\delta^{-1}(\mathcal{U}) \cap R_{\leq D}^{n \times n}$ is Zariski dense in the affine space $R_{\leq D}^{n \times n}$.
\end{lemma}

\begin{proof}
From Hilbert's irreducibility theorem (Theorem \ref{thm.Hilbert's irreducibility theorem}), the set
\[ \{ (c_{\alpha,i,j}) \in (K^{N_D})^{n \times n} \mid \text{$\mathrm{DET}(\{c_{\alpha,i,j}\},\{x_k\})$ is non-zero irreducible} \}, \]
which is equal to $\delta^{-1}(\mathcal{U}) \cap R_{\leq D}^{n \times n}$,
is Zariski dense in the affine space $R_{\leq D}^{n \times n}$.
\end{proof}

Let us summarize our situation in the following commutative diagram:

\[ \xymatrix{
\tilde{\mathcal{F}}_{\leq D}  & \mathcal{X}_{\leq D} \ar[l]_{\varphi} \ar[r]^{p} & R_{\leq D}^{n \times n} \\
\mathcal{F}_0 \cap \tilde{\mathcal{F}}_{\leq D} \ar@{^{(}-_>}[u]& (\delta \circ p)^{-1}(\mathcal{U}) \cap \mathcal{X}_{\leq D} \ar[l] \ar@{^{(}-_>}[u] \ar[r] & \delta^{-1}(\mathcal{U}) \cap R_{\leq D}^{n \times n}. \ar@{^{(}-_>}[u]_{\text{dense}}
}
\]
Here we denote by
\begin{align*}
    \mathcal{X_{\leq D}} &:= \left\{ (B,A) \in R_{\leq D}^{n \times m} \times R_{\leq D}^{m \times n} \left| \begin{aligned} &(BA - E_n)G = O_{n \times 1},\\ &B_1A_1 \in R_{\leq D}^{n \times n} \end{aligned} \right. \right\},\\
\tilde{\mathcal{F}}_{\leq D} &:= \{ F \in R^{m \times 1} \mid \langle F \rangle = \langle G \rangle,\ \exists\, (B,A) \in \mathcal{X}_{\leq D}^{m \times n}\ F = AG \}.    
\end{align*}

Thus, if the set $(\delta \circ p)^{-1}(\mathcal{U}) \cap \mathcal{X}_{\leq D}$ is dense in $\mathcal{X}_{\leq D}$, then the set $\mathcal{F}_0 \cap \tilde{\mathcal{F}}_{\leq D}$ is dense in $\tilde{\mathcal{F}}_{\leq D}$. We use the following lemma that states that the preimage of a dense subset under a flat morphism is still dense.

\begin{lemma}\label{lem:density_of_preimage}
Let $f:Y \rightarrow Z$ be a flat morphism of finite type between Noetherian schemes. Let $\mathcal{V}$ be a dense subset in $Z$. Then the preimage $f^{-1}(\mathcal{V})$ is dense in $Y$.
\end{lemma}

\begin{proof}
It is enough to show that for any non-empty open subset $U$ in $Y$, the intersection $U \cap f^{-1}(\mathcal{V})$ is not empty. In fact, a flat morphism of finite type between Noetherian schemes is open \cite[Exercise III.9.1]{Har77}. Since $f$ is open, the image $f(U)$ is a non-empty open subset in $Z$ for any non-empty open subset $U$ in $Y$. Then we have $f(U) \cap \mathcal{V} \neq \emptyset$, which forces $U \cap f^{-1}(\mathcal{V})\neq \emptyset$.
\end{proof}





\begin{theorem}\label{thm:density_of_outputs}
Assume that 
\[Y = \{ x \in \mathcal{X}_{\leq D}  \mid \text{$p$ is flat at $x$}\}\] is not empty and $\mathcal{X}_{\leq D}$ is irreducible. Then $(\delta \circ p)^{-1}(\mathcal{U}) \cap \mathcal{X}_{\leq D}$ is dense in $\mathcal{X}_{\leq D}$.
\end{theorem}

\begin{proof}
Since $R_{\leq D}^{n \times n}$ is an integral scheme, the flat locus $Y$ is open in $\mathcal{X}_{\leq D}$. In particular, the restriction $p\vert_Y: Y \rightarrow R_{\leq D}^{n \times n}$ is a flat morphism of finite type between Noetherian schemes. By Lemma \ref{lem:density_of_preimage}, the preimage
\[ p^{-1}(\delta^{-1}(\mathcal{U}) \cap R_{\leq D}^{n \times n}) \cap Y\]
is dense in $Y$. Now let us show the density of $(\delta \circ p)^{-1}(\mathcal{U}) \cap \mathcal{X}_{\leq D}$ in $\mathcal{X}_{\leq D}$. For any non-empty open subset $\mathcal{W}$ in $\mathcal{X}_{\leq D}$, the intersection $Y \cap \mathcal{W}$ is a non-empty open subset in $Y$ by our assumption. Then the intersection
\[ (Y \cap \mathcal{W}) \cap (p^{-1}(\delta^{-1}(\mathcal{U}) \cap R_{\leq D}^{n \times n}) \cap Y) = Y \cap \mathcal{W} \cap (\delta \circ p)^{-1}(\mathcal{U})\]
is not empty. In particular,
\[ (\delta \circ p)^{-1}(\mathcal{U}) \cap \mathcal{X}_{\leq D} \cap \mathcal{W} \supset  (\delta \circ p)^{-1}(\mathcal{U}) \cap Y \cap \mathcal{W} \]
is also not empty. Therefore $(\delta \circ p)^{-1}(\mathcal{U}) \cap \mathcal{X}_{\leq D}$ is dense in $\mathcal{X}_{\leq D}$.
\end{proof}
The main result of Section \ref{sec:problems} and of this document can then be obtained as a corollary. 
\begin{corollary}[Theorem \ref{thm:main}]\label{cor:m2n}
If $m \geq 2n \geq 3$ and $\mathcal{X}_{\leq D}$ is irreducible, then the set $\mathcal{F}_0 \cap \tilde{\mathcal{F}}_{\leq D}$ is dense in $\tilde{\mathcal{F}}_{\leq D}$.
\end{corollary}

\begin{proof}[Proof. (also proof of Theorem \ref{thm:main})]
It is enough to show that $Y$ in Theorem \ref{thm:density_of_outputs} is not empty. First we shall construct a section  $\iota : R_{\leq D}^{n \times n} \rightarrow \mathcal{X}_{\leq D}$ of $p: \mathcal{X}_{\leq D} \rightarrow R_{\leq D}^{n \times n}$ such that $p \circ \iota = \mathrm{id}$. For any $C \in R_{\leq D}^{n \times n}$, let us consider matrices
\[ B_2 = (E_n \mid O_{n \times (m-2n)}) \in R^{n \times (m-n)},\ A_2 = \left( \begin{array}{c} E_n - C \\ \hline O_{(m-2n) \times n} \end{array} \right) \in R^{(m-n) \times n} \]
and put
\[ B = (E_n \mid B_2),\ A = \left( \begin{array}{c} C \\ \hline A_2 \end{array} \right). \]
Then we have
\[ E_n - B_2 A_2 = E_n - (E_n-C) = C,\ BA = C+B_2A_2 = E_n. \]
By defining $\iota(C) := (B,A)$, the pair $(B,A)$ is an element of $\mathcal{X}_{\leq D}$ and it holds that $p(B,A) = C$. Clearly, the map $\iota: R_{\leq D}^{n \times n} \rightarrow \mathcal{X}_{\leq D}$ is a morphism of schemes.

Now, we show that the flat locus $Y$ is not empty. Let us consider $R_{\leq D}^{n \times n}$ to be the affine scheme $\spec K[c_{\alpha,i,j} \mid x^{\alpha} \in R_{\leq D},\ 1 \leq i,j \leq n]$. Let $L$ be the field of fractions of $K[c_{\alpha,i,j} \mid x^{\alpha} \in R_{\leq D},\ 1 \leq i,j \leq n]$. Then $\spec L$ is isomorphic to a non-empty open subscheme $\mathcal{V}$ in $R_{\leq D}^{n \times n}$. Since any module over $L$ is flat, the preimage $p^{-1}(\mathcal{V})$ is contained in $Y$. Therefore the surjectivity of $p$ implies that $Y$ is not empty.
\end{proof}

\begin{remark}\label{rem:nonempty_flat_locus}
Similar to the proof of Corollary \ref{cor:m2n}, in any case, there exists a non-empty open dense subscheme $\mathcal{V} \subset R_{\leq D}^{n \times n}$ so that
\[ p : p^{-1}(\mathcal{V}) \cap \mathcal{X}_{\leq D} \rightarrow R_{\leq D}^{n \times n} \]
is flat.
If the flat locus $Y$ is empty, then the image $p(X_{\leq D})$ is in the closed subscheme $R_{\leq D}^{n \times n} \setminus \mathcal{V}$, which is very rare case where $p$ is a generic morphism of schemes and $\mathcal{V}$ is a generic open dense subscheme. This is one reason we assumed Heuristic \ref{heu:chi_is_irr_intro}.
\end{remark}

\bibliographystyle{amsalpha}
\bibliography{main}

\end{document}

%% file: math_commands.tex
\usepackage{amsmath,amsfonts,bm,amsthm}
\usepackage{thmtools}
\usepackage{thm-restate}
\usepackage{physics}

\theoremstyle{plain} \numberwithin{equation}{section}
\newtheorem{theorem}{Theorem}[section]
\numberwithin{theorem}{section}
\newtheorem{lemma}[theorem]{Lemma}
\newtheorem{corollary}[theorem]{Corollary}
\newtheorem{proposition}[theorem]{Proposition}
\newtheorem{problem}[theorem]{Problem}

\theoremstyle{definition}
\newtheorem{definition}[theorem]{Definition}

\newtheorem{remark}[theorem]{Remark}

\theoremstyle{plain}

\usepackage[ruled, linesnumbered]{algorithm2e} %
\usepackage{xcolor} %
\usepackage{tikz}
\usetikzlibrary{fit,calc}

\colorlet{pink}{red!40}
\colorlet{lightblue}{blue!30}
\colorlet{lightgreen}{green!30}

\DontPrintSemicolon